\documentclass[runningheads]{llncs}
 
\usepackage{graphicx}
\usepackage{booktabs}
\usepackage{amsmath}
\usepackage{amssymb}
\usepackage{subcaption} 
\usepackage{multirow} 
\usepackage{algorithm}
\usepackage{algpseudocode}
\usepackage{multicol}
\usepackage{xcolor}
\usepackage[accsupp]{axessibility}  

\usepackage{orcidlink}

\begin{document}

\title{LCM: Log Conformal Maps for Robust Representation Learning to Mitigate Perspective Distortion} 

\titlerunning{LCM}

\author{
Meenakshi Subhash Chippa\inst{1,*}\orcidlink{0009-0000-2770-6271} \and
Prakash Chandra Chhipa\inst{1}\orcidlink{0000-0002-6903-7552} \and
Kanjar De\inst{2}\orcidlink{0000-0003-0221-8268} \and
Marcus Liwicki\inst{1}\orcidlink{0000-0003-4029-6574} \and
Rajkumar Saini\inst{1}\orcidlink{0000-0001-8532-0895}
}

\authorrunning{Chippa et al.}

\institute{Luleå Tekniska Universitet, Sweden \\
\email{\{prakash.chandra.chhipa, rajkumar.saini, marcus.liwicki\}@ltu.se} \\
\email{*meechi-2@student.ltu.se} \and
 Fraunhofer Heinrich-Hertz-Institut, Berlin, Germany \\
\email{kanjar.de@hhi.fraunhofer.de}}

\maketitle
\begin{abstract}

Perspective distortion (PD) leads to substantial alterations in the shape, size, orientation, angles, and spatial relationships of visual elements in images. Accurately determining camera intrinsic and extrinsic parameters is challenging, making it hard to synthesize perspective distortion effectively.
The current distortion correction methods involve removing distortion and learning vision tasks, thus making it a multi-step process, often compromising performance. Recent work leverages the Möbius transform for mitigating perspective distortions (MPD) to synthesize perspective distortions without estimating camera parameters. Möbius transform requires tuning multiple interdependent and interrelated parameters and involving complex arithmetic operations, leading to substantial computational complexity. To address these challenges, we propose Log Conformal Maps (LCM), a method leveraging the logarithmic function to approximate perspective distortions with fewer parameters and reduced computational complexity. We provide a detailed foundation complemented with experiments to demonstrate that LCM with fewer parameters approximates the MPD. We show that LCM integrates well with supervised and self-supervised representation learning, outperform standard models, and matches the state-of-the-art performance in mitigating perspective distortion over multiple benchmarks, namely Imagenet-PD, Imagenet-E, and Imagenet-X. Further LCM demonstrate seamless integration with person re-identification and improved the performance. Source code is made publicly available at \href{https://github.com/meenakshi23/Log-Conformal-Maps}{https://github.com/meenakshi23/Log-Conformal-Maps}.

\end{abstract}

\section{Introduction}
\label{sec:intro}
Perspective distortion (PD) is a common issue in real-world imagery, complicating the development of computer vision applications. It arises from factors like camera position, depth, focal length, lens aberrations, and rotation, which affect the projection of 3D scenes onto 2D surfaces \cite{rahman2011efficient}. Accurately estimating these parameters for PD correction is challenging, posing a significant barrier to robust computer vision (CV) methods. 
\begin{figure*}[h]
  \centering
  \includegraphics[width=0.8\columnwidth]{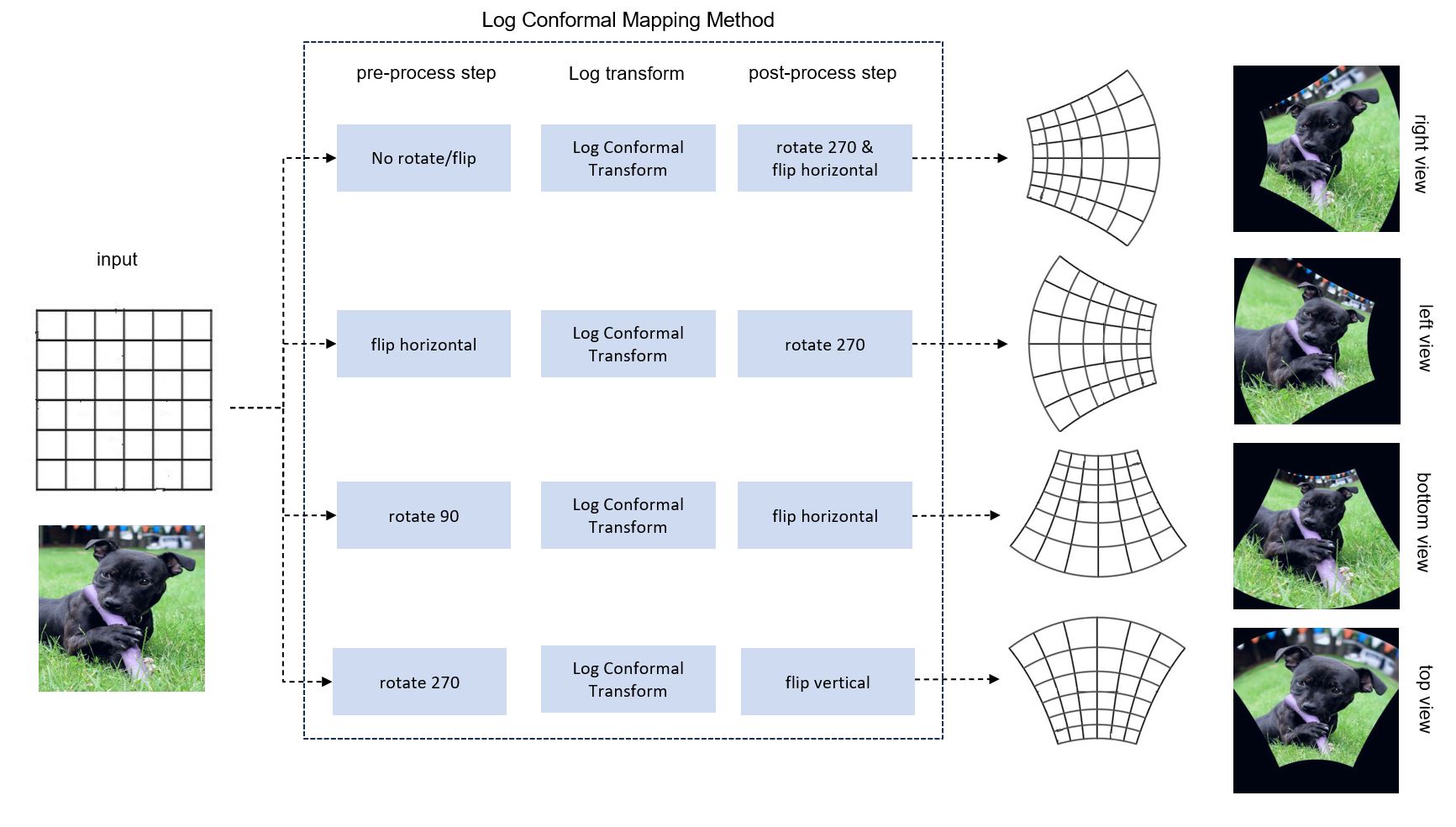}
  \vspace{-3mm}
  \caption{Log conformal Mapping Method (LCM). LCM obtains four perspective distorted views using auxiliary operations with Log Conformal transform.} 
  \label{fig:LCM_fig}
  \vspace{-6mm}
\end{figure*}
Earlier studies focused on distortion correction \cite{li2019blind},  \cite{yu2021pcls},  \cite{tan2021practical}, which makes vision task learning multi-steps. Recent work MPD \cite{chhipa2024m} provides insights on mitigating perspective distortion by synthetically mimicking the PD in the learning process using nonlinear and conformal Möbius \cite{arnold2008mobius} transform.

Previous research on image registration~\cite{vadapally2009image} demonstrates the potential of conformal log-polar mapping. Another study highlights the utility of complex logarithmic views \cite{bottger2006complex} for exploring small details within extensive visual geometry contexts. Focusing on non-linear and conformal properties of log transform, we proposed Log Conformal Mapping (LCM) to mitigate perspective distortion and match state-of-the-art robustness performance. LCM method in Figure \ref{fig:LCM_fig} demonstrates the synthesis of four perspective distorted views (left, right, top, bottom) mentioned in MPD \cite{chhipa2024m}. The main contributions of this paper are:\\

\noindent 1. To the best of our knowledge, this work is the first to employ Log Conformal Transform for synthesizing perspective distortion, introducing the Log Conformal Maps (LCM) method.\\
\noindent 2. We present analytical and empirical evidence demonstrating that LCM, with fewer parameters, effectively mitigates perspective distortion, comparable to the Möbius Transforms-based MPD method.\\
\noindent 2. We validate our approach through extensive experiments on multiple perspective distortion-affected benchmarks, showing improvements in the robustness of both supervised and self-supervised models and showing effectiveness on real-world application, person re-identification.
\section{Related Work}
Perspective distortion is a pervasive issue that significantly impacts the performance of various downstream tasks in computer vision. Perspective distortion arises from the camera's relative positioning to objects in the scene, causing apparent deformation in the captured image. This distortion results from the camera position, depth, intrinsic parameters (e.g., focal length, lens distortion), and extrinsic parameters (e.g., rotation, translation). These elements collectively affect the projection of 3D scenes onto 2D planes, impacting semantic interpretation and local geometry~\cite{rahman2011efficient}. Accurately estimating these parameters for correcting perspective distortion is challenging and remains a critical barrier to developing robust computer vision methods. Another work \cite{cho2021camera} proposes a model that adapts 3D human pose estimation in videos to arbitrary camera distortions using meta-learning and a novel synthetic data generation technique.

Earlier studies on distortion correction primarily focused on rectifying perspective distortion, with limited emphasis on the robustness of computer vision applications. 
These correction methods typically transform computer vision tasks into two-stage processes: initially rectifying the distortion and subsequently engaging in task-specific learning. GeoNet \cite{li2019blind} employs convolutional neural networks (CNNs) to predict distortion flow in images without prior knowledge of the distortion type. Perspective Crop Layers (PCL) \cite{yu2021pcls} use perspective crop layers within CNNs to correct perspective distortions for 3D pose estimation. A cascaded deep structure network \cite{tan2021practical} addresses wide-angle portrait distortions without requiring calibrated camera parameters. Another study \cite{zhao2019learning} predicts per-pixel displacement for face portrait undistortion. PerspectiveNet \cite{jin2023perspective}, and ParamNet \cite{jin2023perspective} predict perspective fields and derive camera methods for camera calibration, respectively. Methods for fisheye image distortion correction have also been developed \cite{yang2023innovating,yin2018fisheyerecnet}.

Earlier works have advanced computer vision tasks under perspective distortion. \cite{wang2023zolly} introduced a method to correct perspective distortions in human mesh reconstruction from images with varying focal lengths. \cite{kocabas2021spec} proposed a framework for estimating camera parameters in the wild to improve 3D human pose and shape estimation. However, these methods rely on an inefficient two-step process and do not focus on robust representations for computer vision tasks.

Recently, the Möbius Transform for Mitigating Perspective Distortions (MPD)~\cite{chhipa2024m} introduced, emphasizing the impact of perspective distortion in real-world computer vision applications such as object recognition and detection. MPD utilizes Möbius transformations, a type of conformal mapping, to synthesize perspective distortion without estimating camera parameters. This method has shown the effectiveness of Möbius transformations in creating controlled distortions that mimic perspective distortion. 
While the Möbius transform effectively mimics perspective distortion, it presents significant challenges. MPD requires precise tuning of four interrelated complex parameters (\(a, b, c, d\)), which are not easily interpretable. This complexity makes the process sub-optimal and time-intensive.
Furthermore, the Möbius transform inherently involves a series of complex arithmetic operations, specifically complex multiplications and divisions, which contribute to its computational intensity. These operations require substantial computational resources, thereby limiting the efficiency of the transformation. Additionally, maintaining the stability of the Möbius transform is particularly challenging near singularities, where the condition \(cz + d = 0\) leads to undefined behavior. These complexities highlight the need for more efficient and robust methods to synthesize perspective distortion, which underscores the need for continued research in this area.
\section{Problem Formulation and Proposed Method}
\label{sec:pfpr}
We begin by formalizing perspective distortion and reviewing the Möbius transform from MPD~\cite{chhipa2024m}, highlighting its challenges. Next, we introduce our novel method, Log Conformal Maps (LCM), and provide proof of its non-linearity and conformity, establishing LCM as a viable candidate for mitigating perspective distortion. Furthermore, we demonstrate the efficiency of LCM by proving theoretically that LCM, with fewer parameters, can approximate the Möbius transform from MPD for mimicking perspective distortion.

\subsection{Mathematical Foundation of Perspective Distortion}
\label{sec:perspective_distortion}
Perspective distortion arises from the projection of three-dimensional (3D) scenes onto a two-dimensional (2D) plane, typically described by the perspective projection model. This projection is inherently non-linear and non-conformal, causing significant visual distortions, especially when objects are viewed from angles other than the perpendicular.Perspective projection can be expressed as:
\[
\begin{pmatrix}
x' \\
y' \\
w'
\end{pmatrix}
= 
\begin{pmatrix}
f & 0 & 0 & 0 \\
0 & f & 0 & 0 \\
0 & 0 & 1 & 0
\end{pmatrix}
\begin{pmatrix}
X \\
Y \\
Z \\
1
\end{pmatrix}
\]
where \((X, Y, Z)\) are the coordinates of a point in 3D space, \((x', y')\) are the coordinates of the projected point on the 2D plane, \(f\) is the focal length, and \(w'\) is a scaling factor given by \(w' = Z\).

The 2D coordinates \((x, y)\) are obtained by normalizing with \(w'\):
\[
x = \frac{x'}{w'} = \frac{fX}{Z}, \quad y = \frac{y'}{w'} = \frac{fY}{Z}
\]

This non-linear relationship between the 3D coordinates and their 2D projections causes perspective distortion, where objects farther from the camera appear smaller and parallel lines seem to converge. Additionally, this projection does not preserve angles (non-conformal), further contributing to the distortion. Affine transformations, which include translations, rotations, scaling, and shearing, can be represented as linear transformations of the form:
\[
\mathbf{A} \mathbf{p} + \mathbf{b}
\]
where \(\mathbf{A}\) is a matrix and \(\mathbf{b}\) is a translation vector. However, affine transformations cannot model the non-linear and non-conformal nature of perspective distortion, as they preserve parallelism and ratios of distances along parallel lines, which are not preserved in perspective projection.

Given these limitations, non-linear transformations Möbius transform from MPD \cite{chhipa2024m} and the proposed Log Conformal Maps, are required to accurately model and mitigate the effects of perspective distortion in computer vision applications.

\subsection{Challenges of the Möbius Transform}
\label{sec:mobius}
The Möbius transformation, also known as a bilinear or fractional linear transformation, is a conformal mapping that preserves angles. It is defined as:
\begin{equation}
\Phi(z) = \frac{az + b}{cz + d}
\end{equation}
where \(a, b, c, d\) are complex numbers and \(ad - bc \neq 0\). This transformation maps the extended complex plane (including the point at infinity) onto itself. The MPD method utilizes the Möbius transformation to synthesize perspective distortions in images. This process involves mapping image coordinates to the complex plane, applying the Möbius transformation, and then mapping the transformed coordinates back to the image plane. 
\subsubsection{sub-optimal}
Despite its effectiveness, the Möbius transform poses several challenges. The transformation requires careful tuning of four interdependent complex parameters \((a, b, c, d)\), each ranging between 0 and 1, creating a parameter space 
\[
\mathcal{P} = \{(a, b, c, d) \in \mathbb{C}^4 \mid ad - bc \neq 0\}.
\]
This space grows exponentially because each parameter can take on numerous values, leading to \(n^4\) possible combinations if discretized into \(n\) steps. The influence of parameters on each other means a change in one parameter significantly affects the others, complicating the optimization process. Determining the parameters, either heuristically or through a learning algorithm, is particularly challenging due to the vast and complex search space required to accurately mimic perspective distortion.

\subsubsection{compute intensive}
The fractional aspect of the Möbius transformation, involving complex division, adds significant computational complexity. Calculating the real and imaginary parts of both the numerator and denominator, followed by their division, requires multiple operations per pixel:

\begin{equation}
\text{Re}(\Phi(z)) = \frac{\text{Re}(az + b) \cdot \text{Re}(cz + d) + \text{Im}(az + b) \cdot \text{Im}(cz + d)}{(\text{Re}(cz + d))^2 + (\text{Im}(cz + d))^2}
\end{equation}
\begin{equation}
\text{Im}(\Phi(z)) = \frac{\text{Im}(az + b) \cdot \text{Re}(cz + d) - \text{Re}(az + b) \cdot \text{Im}(cz + d)}{(\text{Re}(cz + d))^2 + (\text{Im}(cz + d))^2}
\end{equation}

These operations, repeated for each pixel, impose a substantial computational burden, especially for high-resolution images. Additionally, near singularities where \(cz + d \approx 0\), the transformation becomes unstable, leading to very large or undefined values.

\subsection{Proposed Method - Log Conformal Maps (LCM)}
\label{sec:lcm}

To address the challenges posed by the Möbius transform, we propose the Log Conformal Maps (LCM) method. LCM leverages the logarithmic function to approximate perspective distortions with fewer parameters and reduced computational complexity. With flip and rotation operations, LCM achieves all four perspective-distorted views (left, right, top, bottom) similar to the MPD \cite{chhipa2024m}. The Log Conformal Transform is defined as:
\begin{equation}
\Psi(z) = \log(kz + c)
\end{equation}
where \(k\) and \(c\) are complex numbers. The logarithmic function is non-linear yet conformal, preserving angles and providing a smooth distortion. LCM outlined in Algorithm 1. Proof on LCM non-linearity and conformality in sec. \ref{proof:lcm1}.
\begin{algorithm}
\caption{LCM Method for Perspective Distortion}
\footnotesize
\begin{algorithmic}[1]
\Require{Input image $I$, width parameters $k$, and constant parameters $c$, rotation angles $\theta \in \{0^\circ, 90^\circ, 270^\circ\}$, flip}
\Ensure{Transformed image $I_{LCM}$ representing perspective distortion}
\State $I_{\theta} \gets \text{rotate}(I, \theta) + flip$ \Comment{Rotate image by $\theta$ degrees and flip based on view}
\For{each pixel coordinate $(x, y)$ in $I_{\theta}$}
    \State $z \gets x + iy$ \Comment{Map pixel coordinates to complex vector}
    \State $z_{\log} \gets \log(kz + c)$ \Comment{Apply Log Conformal Transform}
    \State $x_{\log} + iy_{\log} \gets z_{\log}$ \Comment{Real and imaginary parts as transformed coordinates}
    \State $(x_d, y_d) \gets \text{round}(x_{\log}), \text{round}(y_{\log})$ \Comment{Discretize to nearest pixel values}
    \State $I_{LCM}(x_d, y_d) \gets I_{\theta}(x, y)$ \Comment{Assign pixel value to transformed coordinates}
\EndFor
\State $I_{LCM} \gets \text{rotate}(I_{LCM}, \theta) + flip$ \Comment{Rotate and flip as post processing}
\State $I_{LCM} \gets \text{Padding}(I_{LCM})$ \Comment{Padding for smoother transformed output}
\State \Return $I_{LCM}$
\end{algorithmic}
\label{algo:lcm}
\end{algorithm}

The Log Conformal Maps (LCM) method applies a logarithmic transform \(\Psi(z) = \log(kz + c)\) to complex coordinates \(z = x + iy\), providing non-linear and conformal mapping. Optionally, the method uses simple rotation and flip operations with transform to obtain all four perspective-distorted views: left, right, top, and bottom. This method simplifies parameter tuning with only the width parameter (\(k\), reducing computational complexity as the logarithmic function avoids the intensive arithmetic of fractions required by the Möbius transform.
\subsection{Non-linearity and Conformality of Log Conformal Maps}
\label{proof:lcm1}

\begin{theorem}
The Log Conformal Map (LCM) defined by \(\Psi(z) = \log(kz + c)\), where \(k\) and \(c\) are complex numbers and \(z\) is a complex variable, is both non-linear and conformal.
\end{theorem}

\begin{proof}
To establish that \(\Psi(z) = \log(kz + c)\) is non-linear and conformal, we need to demonstrate two properties:

\noindent1. \textbf{Non-linearity}: The function \(\Psi(z)\) do not satisfy the superposition principle.\\
2. \textbf{Conformality}: The function \(\Psi(z)\) is holomorphic with a non-zero derivative, ensuring angle preservation.

\paragraph{Non-linearity}
A function \(f(z)\) is non-linear if it does not satisfy the superposition principle:
\[
\Psi(\alpha z_1 + \beta z_2) \neq \alpha \Psi(z_1) + \beta \Psi(z_2)
\]
for complex numbers \(z_1\), \(z_2\) and scalars \(\alpha\), \(\beta\). Consider two complex numbers \(z_1\) and \(z_2\):
\[
\Psi(z_1 + z_2) = \log(k(z_1 + z_2) + c)
\]
However,
\[
\Psi(z_1) + \Psi(z_2) = \log(kz_1 + c) + \log(kz_2 + c) = \log((kz_1 + c)(kz_2 + c))
\]

Certianly, \(\log(k(z_1 + z_2) + c) \neq \log((kz_1 + c)(kz_2 + c))\). Therefore, \(\Psi(z)\) is non-linear.

\paragraph{Conformality}
A function \(f(z)\) is conformal if it preserves angles locally, which is ensured if the function is holomorphic (complex differentiable) with a non-zero derivative.

\noindent1. \textbf{Holomorphicity}: 
   \[
   \Psi(z) = \log(kz + c)
   \]
   To show \(\Psi(z)\) is holomorphic, we need to demonstrate it is complex differentiable. The derivative of \(\Psi(z)\) is:
   \[
   \Psi'(z) = \frac{d}{dz} \log(kz + c) = \frac{k}{kz + c}
   \]

\noindent2. \textbf{Non-zero Derivative}:
   For \(\Psi(z)\) to be conformal, its derivative must be non-zero. Since \(k \neq 0\) and \(kz + c \neq 0\) for all \(z \in \mathbb{C}\), the derivative \(\Psi'(z) = \frac{k}{kz + c}\) is non-zero.

Therefore, since \(\Psi(z)\) is holomorphic with a non-zero derivative, it is conformal and preserves angles locally.

Combining these results, we conclude that the Log Conformal Map (LCM), \(\Psi(z) = \log(kz + c)\), is both non-linear and conformal.

\end{proof}

\section{Results and Discussions}

We trained ResNet50~\cite{he2016deep} for supervised learning using the protocol outlined in \cite{pytorch_vision}, and SimCLR~\cite{chen2020simple} for self-supervised contrastive learning on the ImageNet~\cite{deng2009imagenet} dataset. Additionally, the self-supervised method DINO~\cite{caron2021emerging} was trained following the authors' specified protocol. The term "Standard ResNet50" refers to the ResNet50 model trained on ImageNet as guided by \cite{pytorch_vision}, while "DINO" refers to the original works of SimCLR~\cite{chen2020simple} and DINO~\cite{caron2021emerging}. We evaluated our methods on three public datasets: Imagenet-PD~\cite{chhipa2024m}, explicitly introduced to benchmark images distorted by perspective distortion, and other perspective distortion-affected benchmarks, ImageNet-E~\cite{li2023imagenet} and Imagenet-X~\cite{idrissi2022imagenet}. The results of these evaluations are presented in this section.

\subsection{ImageNet-PD}
ImageNet-PD~\cite{chhipa2024m} was introduced to benchmark robustness to perspective distortion. This dataset is derived from ImageNet classes and consists of four subsets, each corresponding to a different direction—left, right, top, and bottom—mimicking perspective distortion.

\subsubsection{Supervised learning}
\label{sec:sup}
In Table~\ref{tab:supervised}, we provide results for ImageNet-PD. The hyper-parameter \(p\) represents the probability of applying LCM during the training process. From Table~\ref{tab:supervised}, it is evident that when training performance with standard protocol \cite{pytorch_vision}, there is a significant drop in accuracy. However, LCM adapted training results in a substantial boost in accuracy across all subsets of the ImageNet-PD dataset.
\begin{table}[ht]
\caption{LCM trained on supervised learning and evaluated on \textbf{ImageNet-PD subsets} and original ImageNet. The probability \(P\) of applying LCM in model fine-tuning. results for probability \(P\)=0 denotes standard training and reported from MPD \cite{chhipa2024m}. PD - Perspective Distorted.}
\label{tab:supervised}
\tiny
\begin{tabular}{ccccccccccc}
\hline
\multicolumn{1}{c|}{P} & \multicolumn{2}{c|}{Original ImageNet} & \multicolumn{2}{c|}{\begin{tabular}[c]{@{}c@{}} Top-view (PD-T)\end{tabular}} & \multicolumn{2}{c|}{\begin{tabular}[c]{@{}c@{}} Bottom-view (PD-B)\end{tabular}} & \multicolumn{2}{c|}{\begin{tabular}[c]{@{}c@{}} Left-view (PD-L)\end{tabular}} & \multicolumn{2}{c}{\begin{tabular}[c]{@{}c@{}} Right-view (PD-R)\end{tabular}} \\ \hline
\multicolumn{1}{c|}{} & Top1 & \multicolumn{1}{c|}{Top 5} & Top1 & \multicolumn{1}{c|}{Top 5} & Top1 & \multicolumn{1}{c|}{Top 5} & Top1 & \multicolumn{1}{c|}{Top 5} & Top1 & Top 5 \\ \hline

\multicolumn{1}{c|}{0.0} & 76.13±0.04 & \multicolumn{1}{c|}{92.86±0.01} & 63.37±0.06 & \multicolumn{1}{c|}{83.61±0.02} & 61.15±0.04 & \multicolumn{1}{c|}{81.86±0.01} & 65.20±0.03 & \multicolumn{1}{c|}{85.13±0.03} & 65.84±0.06 & 85.64±0.02 \\ \hline
\multicolumn{1}{c|}{0.2} & 75.96±0.02 & \multicolumn{1}{c|}{92.30±0.04} & 71.94±0.01 & \multicolumn{1}{c|}{90.00±0.02} & 71.88±0.03 & \multicolumn{1}{c|}{90.74±0.03} & 71.92±0.02 & \multicolumn{1}{c|}{90.78±0.02} & 71.95±0.02 & 90.95±0.02 \\
\multicolumn{1}{c|}{0.4} & 76.08±0.02 & \multicolumn{1}{c|}{92.98±0.04} & 72.44±0.02 & \multicolumn{1}{c|}{90.40±0.02} & 72.00±0.03 & \multicolumn{1}{c|}{91.02±0.02} & 72.88±0.04 & \multicolumn{1}{c|}{91.20±0.03} & 72.58±0.02 & 91.22±0.02 \\
\multicolumn{1}{c|}{0.6} & 76.18±0.04 & \multicolumn{1}{c|}{92.80±0.04} & 72.68±0.03 & \multicolumn{1}{c|}{90.52±0.02} & 72.04±0.03 & \multicolumn{1}{c|}{90.88±0.03} & 73.10±0.03 & \multicolumn{1}{c|}{91.48±0.03} & 72.76±0.04 & 91.05±0.03 \\
\multicolumn{1}{c|}{0.8} & 76.22±0.03 & \multicolumn{1}{c|}{92.96±0.03} & 72.90±0.04 & \multicolumn{1}{c|}{90.58±0.03} & 71.95±0.02 & \multicolumn{1}{c|}{91.02±0.02} & 73.22±0.02 & \multicolumn{1}{c|}{91.10±0.04} & 72.83±0.03 & 91.16±0.03 \\ 
\multicolumn{1}{l|}{1.0} & 74.30±0.02 & \multicolumn{1}{c|}{91.10±0.02} & 70.98±0.03 & \multicolumn{1}{c|}{89.90±0.02} & 71.57±0.02 & \multicolumn{1}{c|}{90.86±0.03} & 71.92±0.03 & \multicolumn{1}{c|}{91.00±0.02} & 72.10±0.03 & 90.30±0.03 \\ \hline
\end{tabular}

\end{table}
From Figure~\ref{fig:imagnet_pd_lcm_mpd_compare_supervised}, we observe that a fully supervised LCM adapted ImageNet trained model shows higher accuracy than standard ImageNet trained ResNet50 \cite{pytorch_vision} and retains similar performance with highly compute-intensive MPD \cite{chhipa2024m}. 
\begin{figure*}[h]
  \centering
  \includegraphics[width=\columnwidth]{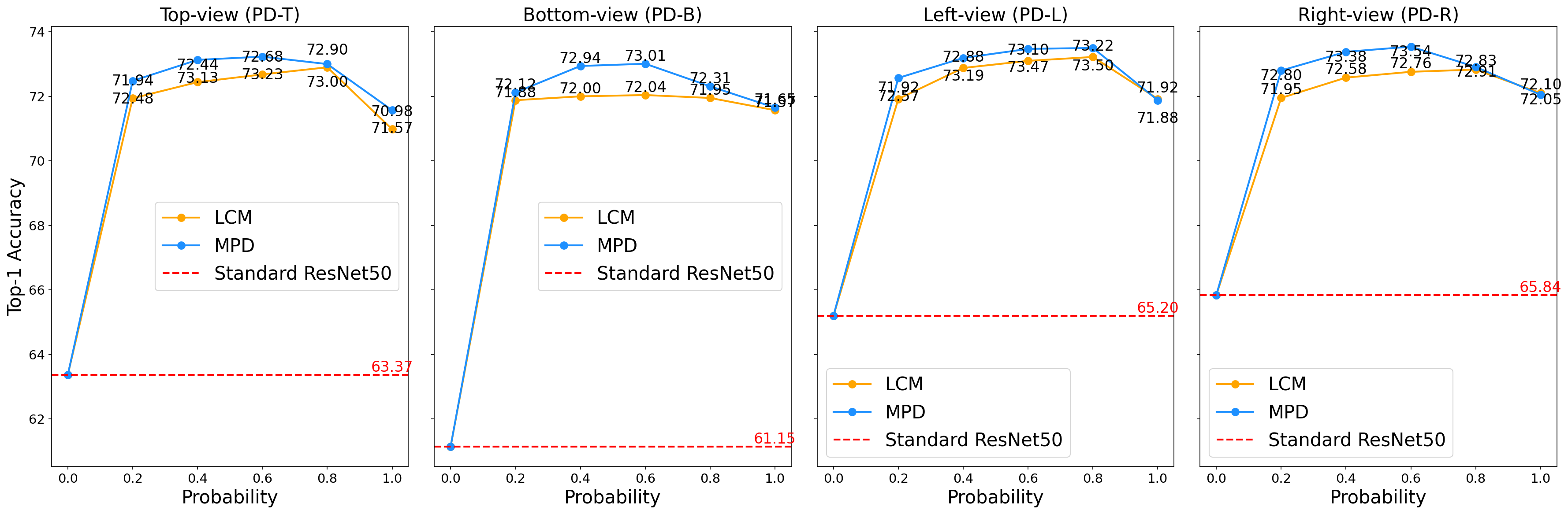}
  \caption{Comparison of LCM with MPD method \cite{chhipa2024m} across probability values. Red line shows standard ResNet50. LCM outperform ResNet50 model trained on ImageNet and matches the performance with MPD. MPD results from \cite{chhipa2024m}.} 
  \label{fig:imagnet_pd_lcm_mpd_compare_supervised}
  \vspace{-4mm}
\end{figure*}
We also compare the LCM with other popular data augmentation methods in Table \ref{tab:imagenetaugs} on ImageNet-PD, where LCM demonstrates a significant performance lead over other methods in supervised approach.
\begin{table}[htp]

\centering
\tiny
\caption{Comparisons with other augmentation methods.}
\label{tab:imagenetaugs}

\begin{tabular}{c|c|cccc}
\hline
\multirow{3}{*}{Method}                               & \multirow{2}{*}{\begin{tabular}[c]{@{}c@{}}Original \\ ImageNet\end{tabular}} & \multicolumn{4}{c}{ImageNet PD}                                                                                                                                                                                                                                                                                       \\ \cline{3-6}
                                                      &                                                                              & \multicolumn{1}{c|}{\begin{tabular}[c]{@{}c@{}}Top-view \\ (PD-T)\end{tabular}} & \multicolumn{1}{c|}{\begin{tabular}[c]{@{}c@{}}Bottom-view\\  (PD-B)\end{tabular}} & \multicolumn{1}{c|}{\begin{tabular}[c]{@{}c@{}}Left-view \\ (PD-L)\end{tabular}} & \begin{tabular}[c]{@{}c@{}}Right-view \\ (PD-R)\end{tabular} \\ \cline{2-6} 
                                                      & Top1                                                                         & Top1                                                                       & Top1                                                                          & Top1                                                                        & Top1                                                    \\ \hline
standard ResNet50                                     & 76.13                                                                        & 63.37                                                                     & 61.15                                                                        & 65.20                                                                      & 65.84                                                  \\
+Mixup \cite{zhang2018mixup}                                               & 77.46                                                                        & 65.46                                                                     & 66.79                                                                        & 68.02                                                                      & 68.43                                                  \\
+Cutout \cite{devries2017improved}                                              & 77.08                                                                        & 64.27                                                                     & 62.04                                                                        & 65.45                                                                      & 65.61                                                  \\
+AugMix \cite{hendrycksaugmix}                                               & 77.53                                                                        & 64.12                                                                     & 62.90                                                                        & 65.95                                                                      & 66.49                                                  \\
+Pixmix \cite{hendrycks2022pixmix}                                              & 77.37                                                                        & 65.52                                                                     & 64.76                                                                        & 67.26                                                                      & 67.56                                                  \\ \hline
LCM                                                  & 76.22                                                                        & \textbf{72.90}                                                            & \textbf{71.95}                                                               & \textbf{73.22}                                                             & \textbf{72.83}                                         \\
\begin{tabular}[c]{@{}c@{}}LCM + SimCLR\end{tabular} & 76.60                                                                        & \textbf{73.50}                                                            & \textbf{73.40}                                                               & \textbf{73.44}                                                             & \textbf{73.60}                                         \\ \hline
\end{tabular}

\vspace{-3mm}
\end{table}


\subsubsection{Self-supervised Representation learning}
We investigate the robustness of LCM in the SimCLR self-supervised learning method, where pretraining is performed with LCM added as an augmentation with a probability of 0.8. Fine-tuning is then conducted following the protocol outlined in Section \ref{sec:sup}. Results for the self-supervised representation are provided in Table~\ref{tab:simclr_finetune}. As shown in Figure~\ref{fig:imagnet_pd_lcm_mpd_compare_simclr}, LCM-trained ImageNet weights, when fully fine-tuned, outperform the standard ImageNet training and demonstrate competitive performance compared to the highly compute-intensive MPD method. Similar to the supervised approach, self-supervised trained LCM models also showcase significant performance improvement compared to other popular data augmentations (Table \ref{tab:imagenetaugs}). 
\begin{table}[h]
\vspace{-8mm}
\caption{LCM trained with SimCLR\cite{chen2020simple} self-supervised pretraining with probability $0.8$ and evaluated on \textbf{ImageNet-PD subsets} and original ImageNet. The probability \(P\) of applying LCM in model fine-tuning stage. Results for probability \(P\) denotes standard training from MPD \cite{chhipa2024m}. PD - Perspective Distorted.}
\centering
\tiny
\label{tab:simclr_finetune}
\begin{tabular}{ccccccccccc}
\hline
\multicolumn{1}{c|}{P} & \multicolumn{2}{c|}{Original ImageNet} & \multicolumn{2}{c|}{\begin{tabular}[c]{@{}c@{}} Top-view (PD-T)\end{tabular}} & \multicolumn{2}{c|}{\begin{tabular}[c]{@{}c@{}} Bottom-view (PD-B)\end{tabular}} & \multicolumn{2}{c|}{\begin{tabular}[c]{@{}c@{}} Left-view (PD-L)\end{tabular}} & \multicolumn{2}{c}{\begin{tabular}[c]{@{}c@{}} Right-view (PD-R)\end{tabular}} \\ \hline
\multicolumn{1}{c|}{} & Top1 & \multicolumn{1}{c|}{Top 5} & Top1 & \multicolumn{1}{c|}{Top 5} & Top1 & \multicolumn{1}{c|}{Top 5} & Top1 & \multicolumn{1}{c|}{Top 5} & Top1 & Top 5 \\ \hline
\multicolumn{1}{c|}{0.2}                & 76.44±0.03 & \multicolumn{1}{c|}{93.42±0.01} & 72.25±0.03                              & \multicolumn{1}{c|}{90.90±0.01}                              & 72.36±0.04                                & \multicolumn{1}{c|}{90.82±0.03}                               & 72.36±0.03                               & \multicolumn{1}{c|}{91.26±0.01}                              & 72.60±0.02                                         & 91.22±0.02                                         \\
\multicolumn{1}{c|}{0.4}                & 76.60±0.04 & \multicolumn{1}{c|}{93.12±0.03} & 73.12±0.02                              & \multicolumn{1}{c|}{91.20±0.02}                              & 73.12±0.02                                & \multicolumn{1}{c|}{91.60±0.03}                               & 73.48±0.02                               & \multicolumn{1}{c|}{91.60±0.02}                              & 73.42±0.03                                         & 91.80±0.02                                         \\
\multicolumn{1}{c|}{0.6}                & 76.08±0.02 & \multicolumn{1}{c|}{93.20±0.04} & 73.10±0.04                              & \multicolumn{1}{c|}{91.80±0.04}                              & 73.20±0.04                                & \multicolumn{1}{c|}{91.40±0.03}                               & 73.42±0.04                               & \multicolumn{1}{c|}{91.55±0.01}                              & 73.52±0.02                                         & 91.62±0.04                                         \\
\multicolumn{1}{c|}{0.8}                & 76.60±0.02 & \multicolumn{1}{c|}{92.96±0.04} & 73.50±0.02                              & \multicolumn{1}{c|}{91.22±0.02}                              & 73.40±0.03                                & \multicolumn{1}{c|}{91.76±0.04}                               & 73.44±0.03                               & \multicolumn{1}{c|}{91.78±0.03}                              & 73.60±0.01                                         & 91.66±0.04                                         \\
\multicolumn{1}{l|}{1.0}                & 74.38±0.02 & \multicolumn{1}{c|}{91.35±0.02} & 71.45±0.02                              & \multicolumn{1}{c|}{90.20±0.02}                              & 71.55±0.02                                & \multicolumn{1}{c|}{90.40±0.03}                               & 71.88±0.02                               & \multicolumn{1}{c|}{90.56±0.04}                              & 72.00±0.04                                         & 90.64±0.03                                         \\ \hline
\end{tabular}
\vspace{-2mm}
\end{table}

\begin{figure*}[!t]
  \centering
  \includegraphics[width=\columnwidth]{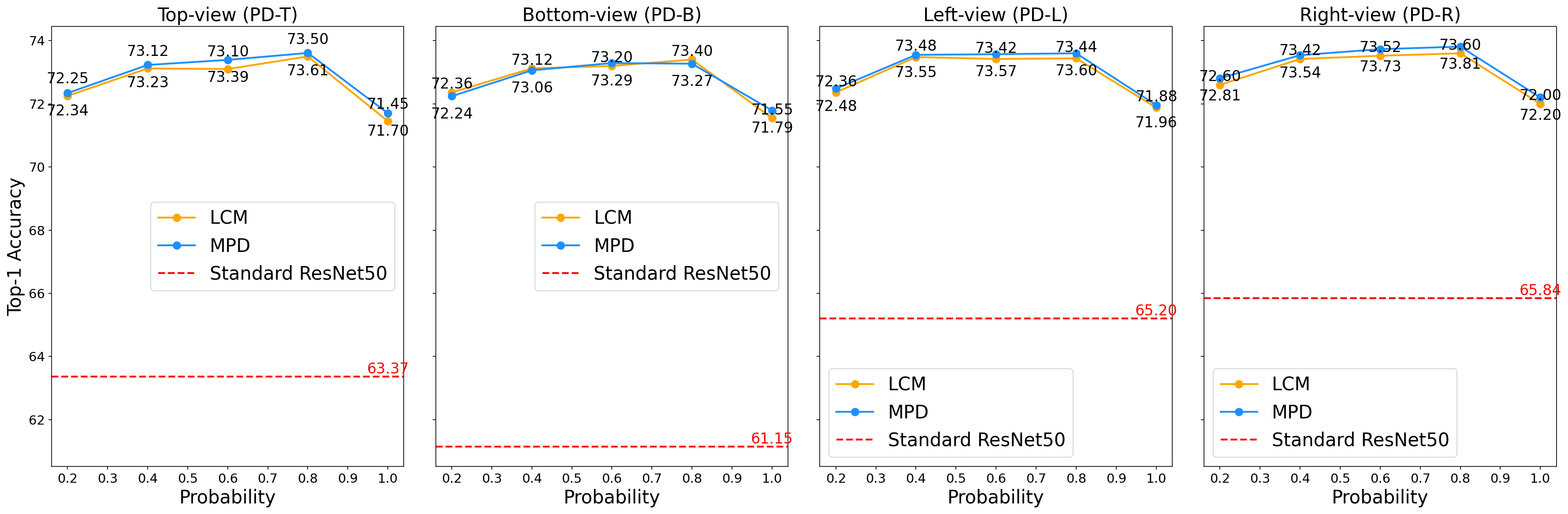}
  \caption{Comparison of self-supervised trained LCM+SimCLR with MPD+SimCLR method \cite{chhipa2024m} across probability values. Red line shows standard ResNet50. LCM outperform ResNet50 model trained on ImageNet and matches the performance with MPD. MPD results are reported from \cite{chhipa2024m}.} 
  \label{fig:imagnet_pd_lcm_mpd_compare_simclr}
  \vspace{-6mm}
\end{figure*}

\begin{table}[]
\caption{Linear evaluation of SimCLR \cite{chen2020simple} contrastive self-supervised models on \textbf{ImageNet-PD sub-sets}. Self-supervised pretraining was performed on ImageNet dataset with batch size of 512 on linear learning rate for 100 epochs, refer SimCLR\cite{chen2020simple}. The probability of LCM is set to 0.8 for self-supervised pre-training and linear evaluation. ResNet50 is backbone. results for standard SimCLR and SimcLR+MPD are reported from MPD \cite{chhipa2024m}.}
\centering
\label{tab:simclr_linear}
\tiny
\begin{tabular}{c|c|cccc}
\hline
Model          & Original ImageNet & \begin{tabular}[c]{@{}c@{}} Top-view (PD-T)\end{tabular} & \begin{tabular}[c]{@{}c@{}} Bottom-view (PD-B)\end{tabular} & \begin{tabular}[c]{@{}c@{}} Left-view (PD-L)\end{tabular} & \begin{tabular}[c]{@{}c@{}} Right-view (PD-R)\end{tabular} \\ \hline
standard SimCLR & 60.14              & 34.22                                                                        & 32.87                                                                           & 36.33                                                                         & 36.89                                                                         \\ \hline
SimCLR+MPD       & 60.02±0.03              & 50.63±0.03                                                                        & 49.65±0.03                                                                           & 50.41±0.03                                                                         & 50.64±0.05                                                                        \\ \hline
SimCLR+LCM       & \textbf{60.30±0.02}              & \textbf{51.04±0.03}                                                                        & \textbf{50.00±0.03}                                                                           & 49.88±0.04                                                                         & 49.78±0.02                                                                        \\ \hline

\end{tabular}
\end{table}

\begin{table*}[h]
\caption{Linear evaluation of knowledge distillation based self-supervised method DINO \cite{caron2021emerging} on \textbf{ImageNet-PD sub-sets}. Self-supervised pretraining was performed on ImageNet dataset with batch size of 512 on linear learning rate for 100 epochs, refer DINO\cite{caron2021emerging}. The probability of LCM is set to 0.8 for self-supervised pre-training and linear evaluation. Results for standard DINO and DINO+MPD are reported from MPD \cite{chhipa2024m}. ViT-small transformer \cite{dosovitskiy2020image} is common backbone. 
}
\scriptsize
\centering
\label{tab:dino_linear}
\tiny
\begin{tabular}{c|c|cccc}
\hline
Model          & Original ImageNet & \begin{tabular}[c]{@{}c@{}} Top-view (PD-T)\end{tabular} & \begin{tabular}[c]{@{}c@{}} Bottom-view (PD-B)\end{tabular} & \begin{tabular}[c]{@{}c@{}} Left-view (PD-L)\end{tabular} & \begin{tabular}[c]{@{}c@{}} Right-view (PD-R)\end{tabular} \\ \hline
standard DINO & 74.00                   & 46.31                                                                             & 46.05                                                                                & 46.15                                                                              & 45.98                                                                               \\ \hline
DINO+MPD & 72.36±0.01              & 60.72±0.02                                                                        & 60.86±0.02                                                                           & 60.63±0.02                                                                         & 60.58±0.02                                                                          \\ \hline
DINO+LCM & \textbf{72.42±0.02}              & \textbf{61.10±0.02}                                                                        & \textbf{61.06±0.03}                                                                           & 59.85±0.03                                                                         & 60.02±0.01                                                                          \\ \hline
\end{tabular}
\vspace{-3mm}
\end{table*}
We also investigate the linear evaluation performance of self-supervised methods SimCLR~\cite{chen2020simple} and DINO~\cite{caron2021emerging}. Similar to SimCLR, DINO is adapted with LCM as an augmentation, and pretraining is performed, followed by linear evaluation. LCM outperforms the standard original SimCLR~\cite{chen2020simple} and DINO~\cite{caron2021emerging} when integrated and linearly evaluated (see Figure \ref{fig:imagnet_pd_lcm_mpd_compare_dino_simclr}). Detailed results are shown in Tables \ref{tab:simclr_linear} and \ref{tab:dino_linear}. Notably, while LCM-adapted self-supervised models outperform ImageNet-PD subsets, they retain performance on the original ImageNet.
\begin{figure*}[!t]
  \centering
  \includegraphics[width=0.6\columnwidth]{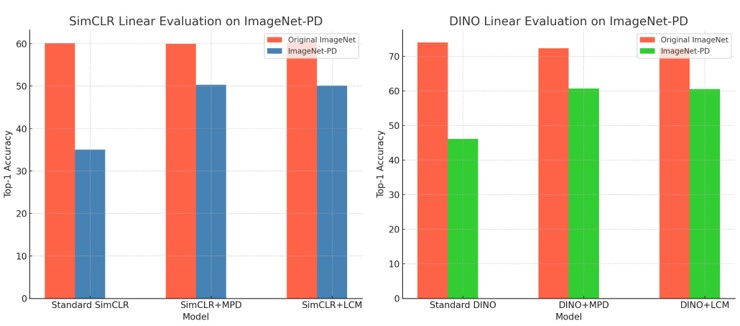}
  \vspace{-2mm}
  \caption{Comparison of linear evaluation performance on self-supervised trained  models. (Left): original SimCLR \cite{chen2020simple}, LCM+SimCLR and MPD+SimCLR\cite{chhipa2024m} on original ImageNet and ImageNet-PD subsets. Performance on ImageNet-PD is average over all four subsets. (Right): Comparing original DINO \cite{caron2021emerging}, LCM+DINO and MPD+DINO \cite{chhipa2024m} on original ImageNet and ImageNet-PD subsets. Performance on ImageNet-PD is average over all subsets.} 
  \label{fig:imagnet_pd_lcm_mpd_compare_dino_simclr}
  \vspace{-8mm}  
\end{figure*}
\subsection{Qualitative Analysis}
In this section, we present qualitative results in Figure~\ref{fig:imagnet_e_cam} to assess the robustness of our proposed formulation. We have plotted the Grad-CAM for the LCM-trained SimCLR method and attention maps for the LCM-trained DINO method, comparing these with the Möbius and LCM methods. The Grad-CAM results for the LCM-trained SimCLR method show clearly defined attention regions, indicating effective feature localization. Similarly, the attention maps for the LCM-trained DINO method demonstrate precise focus areas. These visualizations highlight that LCM, with fewer parameters, produces comparative qualitative outcomes to MPD ~\cite{chhipa2024m}.
\begin{figure*}[!t]
  \centering
  \includegraphics[width=0.8\columnwidth]{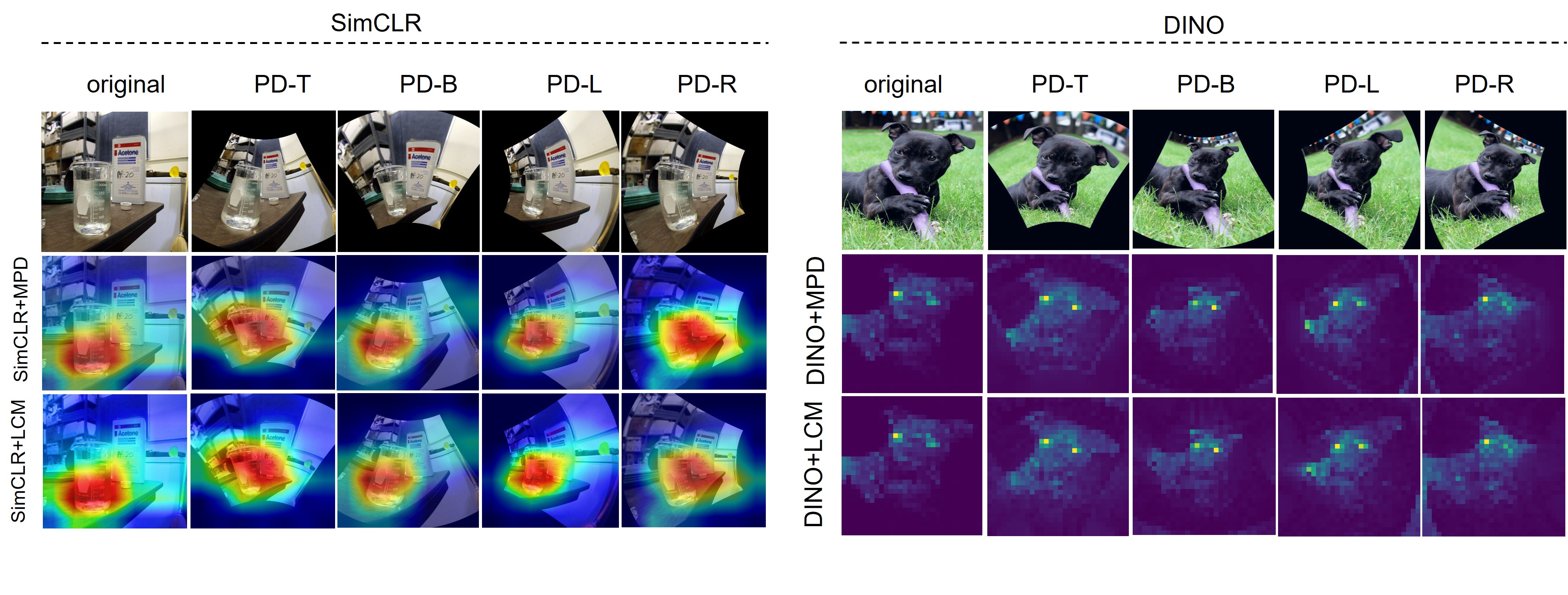}
  \vspace{-3mm}
  \caption{Activation maps: (Left) \textit{beaker} example (n02815834) in ImageNet-PD subsets comparing MPD and LCM in self-supervised learning method SimCLR \cite{chen2020simple}. (Right) \textit{Staffordshire bullterrier} example (n02093256) in ImageNet-PD subsets comparing MPD and LCM in self-supervised learning method DINO \cite{caron2021emerging}. Reported result for MPD is used and same example were used for comparison.} 
  \label{fig:imagnet_e_cam}
  \vspace{-3mm}
\end{figure*}

\subsection{ImageNet-E}
ImageNet-Editing, commonly known as ImageNet-E~\cite{li2023imagenet}, is a comprehensive dataset for benchmarking robustness against various object attributes, including perspective distortion. We extensively evaluated the LCM-trained models, both supervised and self-supervised, and compared their performance with other methods. The comparison are presented in Figure~\ref{fig:imagnet_e_extended} and detailed results are in Table~\ref{tab:imagenet-e_detailed} in suppl. material. 
The main observation is a significant boost in accuracy when LCM is incorporated into the training procedure, benefiting both background and object size changes in the dataset's images. Figure~\ref{fig:imagnet_e_extended} compares our proposed LCM-based technique with the existing Möbius-based MPD~\cite{chhipa2024m} method, demonstrating that the LCM-trained models outperform standard trained models overall the subsets and offers competitive performance compared with MPD \cite{chhipa2024m} with fewer parameters and less compute-intensive, showing resilient for robustness.
\begin{figure*}[t]
  \centering
  \includegraphics[width=1.\columnwidth]{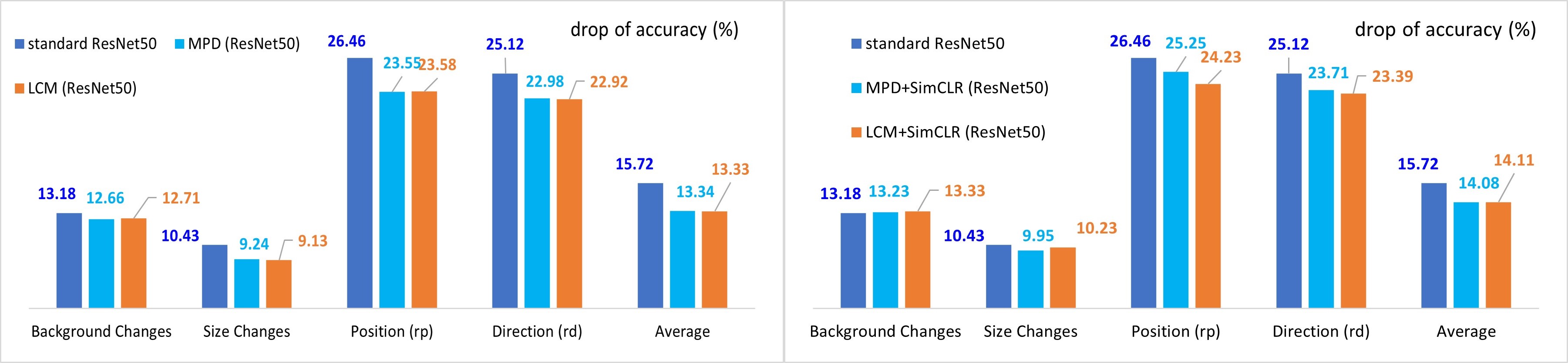}
  \caption{ImageNet-E: Drop of Top1 accuracy under background changes, size changes, random position (rp), random direction (rd), and average over 11 subsets. Lower is better. The mean is reported for size and background-related subsets in ImageNet-E\cite{li2023imagenet}. \textit{Average} reports the mean of all subsets. (Left): compares LCM with MPD \cite{chhipa2024m} in supervised approach. (Right): Compares LCM with MPD \cite{chhipa2024m} in SimCLR self-supervised learning approach.} 
  \label{fig:imagnet_e_extended}
  \vspace{-4mm}
\end{figure*}

\subsection{ImageNet-X}
ImageNet-X~\cite{idrissi2022imagenet} contains human annotations identifying failure types in the widely-used ImageNet dataset. These annotations label distinguishing object factors such as pose, size, color, lighting, occlusions, and co-occurrences for each image in the validation set and a random subset of 12,000 training samples. We evaluated LCM-trained ImageNet models on ImageNet-X, and our findings are reported in Figure~\ref{fig:imagnet_x_all} and detailed results are in Table~\ref{tab:imagenet_x_all} in suppl. material. 
Figure~\ref{fig:imagnet_x_all} illustrates a reduction in the error ratio towards the optimal value of 1.0, decreasing from 1.55 to 1.43 in fully supervised experiments. A similar trend is observed for the self-supervised method, indicating increased robustness.

\begin{figure*}[t]
  \centering
  \includegraphics[width=1.\columnwidth]{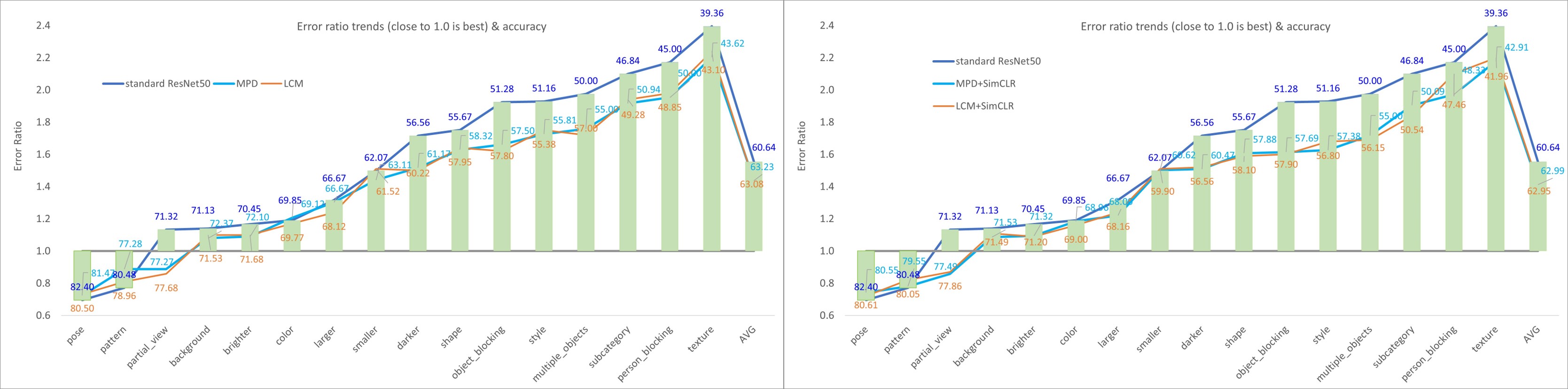}
  \caption{ImageNet-X\cite{idrissi2022imagenet}: Compares error ratio and accuracy of 16 factors for standard ResNet50 model with LCM and MPD. Results for standard ResNet50 ImageNet trained model and MPD trained model are reported from MPD \cite{chhipa2024m}. Error ratio close to 1.0 shows highest robustness \cite{idrissi2022imagenet}. (Left): Compares LCM with standard ResNet50 and MPD in supervised learning. (Right): Compares LCM with standard ResNet50 and MPD in SimCLR \cite{chen2020simple} self-supervised learning.}
  \label{fig:imagnet_x_all}
  \vspace{-5mm}
\end{figure*}

\subsection{LCM Compute Efficiency}
To demonstrate the computational efficiency of the proposed LCM method over the existing Möbius Transform-based MPD method~\cite{chhipa2024m}, we performed a FLOP analysis, with the details presented in Table~\ref{tab:flops_comparison}. Notably, LCM requires fewer floating point operations per second (FLOPs) than the MPD method. The reduced FLOPs are primarily due to LCM's less complex arithmetic operations compared to Möbius transformations. Our benchmarking of the compute time for image transformation operations using LCM and MPD on a CPU-only setup further underscores the efficiency of the LCM method. The results show that MPD requires 0.24 seconds per image per transform, while LCM, with its streamlined approach, needs only 0.22 seconds per image per transform when using a single-core CPU. Table \ref {tab:compute_efficiency_power} shows the compute efficiency of LCM over MPD.
\begin{table}[]
\vspace{-9mm}
\tiny
\centering
\caption{Computation efficiency of LCM}
\label{tab:flops_comparison}
\begin{tabular}{c|c|c|c}
\hline
\textbf{Operation}                          & \textbf{FLOPs per Pixel} & \textbf{MPD}                         & \textbf{LCM}                         \\ \hline
Meshgrid Creation                           & 2                        & $2 \times \text{H} \times \text{W}$  & $2 \times \text{H} \times \text{W}$  \\
Complex Arithmetic for Transformation       & 14                       & $14 \times \text{H} \times \text{W}$ & 0                                    \\
Logarithmic Mapping                         & 4                        & 0                                    & $4 \times \text{H} \times \text{W}$  \\
Scaling and Conversion to Image Coordinates & 6                        & 0                                    & $6 \times \text{H} \times \text{W}$  \\
Grid Sampling (Bilinear Interpolation)      & 7                        & $7 \times \text{H} \times \text{W}$  & $7 \times \text{H} \times \text{W}$  \\ \hline
\textbf{Total FLOPs}                        &                          & $23 \times \text{H} \times \text{W}$ & $19 \times \text{H} \times \text{W}$ \\ \hline
\end{tabular}
\vspace{-7mm}
\end{table}
\begin{table}[h]
\vspace{-9mm}
\centering
\scriptsize
\caption{Compute efficiency and power consumption between MPD and LCM.}
\begin{tabular}{c|c|c}
\hline
\textbf{Metric} & \textbf{MPD} & \textbf{LCM} \\ \hline
\textbf{Time/epoch} & 0.1667 hours & 0.1528 hours \\ 
\textbf{Complete training} & 16.67 hours & 15.28 hours \\ 
\textbf{Saved hours} & - & \textbf{1.39 hrs (8.33\%)} \\ \hline
\textbf{Power/epoch} & 400.08 kWh & 366.72 kWh \\ 
\textbf{Total power} & 4000.8 kWh & 3667.2 kWh \\ 
\textbf{Saved power} & - &\textbf{ 333.6 kWh} \\ \hline
\end{tabular}
\label{tab:compute_efficiency_power}
\vspace{-0.5cm}
\end{table}
\begin{figure}[h]
  \centering
  \begin{minipage}{0.45\columnwidth}
    \centering
    \includegraphics[width=\linewidth]{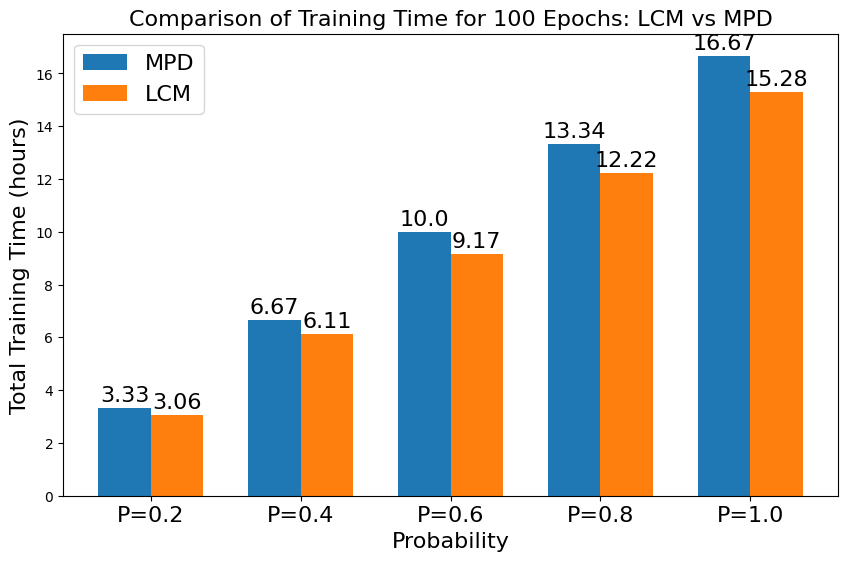}
    \vspace{-6mm}
    \subcaption{}
    \label{fig:hours}
  \end{minipage}%
  \hfill
  \begin{minipage}{0.45\columnwidth}
    \centering
    \includegraphics[width=\linewidth]{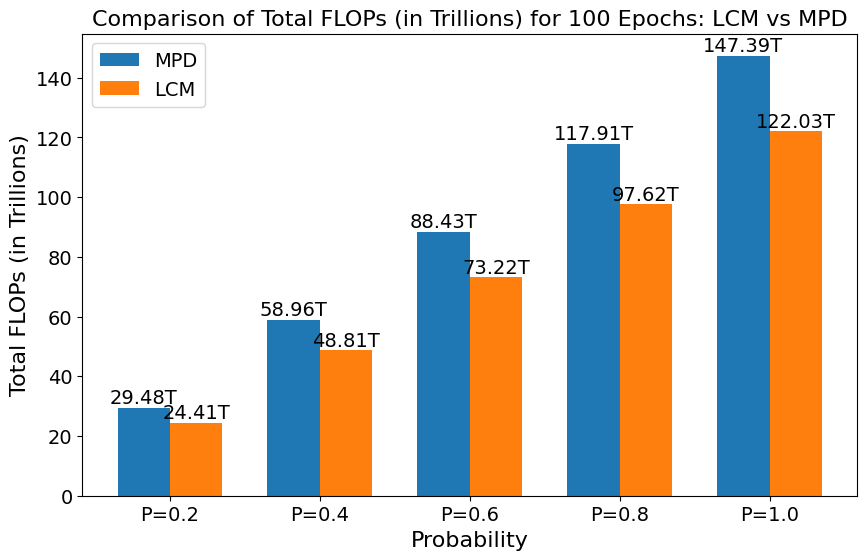}
    \vspace{-6mm}
    \subcaption{}
    \label{fig:flops}
  \end{minipage}
  \vspace{-5mm}
  \caption{Comparison of LCM and MPD for (a) training time efficiency and (b) FLOPs across probability ranges.}
  \label{fig:combined}
  \vspace{-0.5cm}
\end{figure}


To compare the compute efficiency of LCM and MPD methods, we performed a detailed compute analysis with our pretraining experiments using SimCLR self-supervised pretraining on a ResNet50 model with the ImageNet dataset. The model trains for 100 epochs with an input size of 224x224 and a batch size of 512, distributed across 8 NVIDIA H100 GPUs. Both methods show nearly identical per-image processing times (LCM: 0.22s, MPD: 0.24s). LCM demonstrates significant efficiency, saving 1.39 hours (8\% increase) and 333.6 kWh over 100 epochs (Table \ref{tab:compute_efficiency_power} and Fig. \ref{fig:hours}). LCM reduces FLOPs upto 17\% compared to MPD (Fig. \ref{fig:flops}). If method-specific time is discounted, DINO will show a similar trend in compute efficiency.
\section{Person Re-identification} 
Following MPD \cite{chhipa2024m}, we integrated LCM with CLIP-ReIdent \cite{habel2022clip}, a CLIP-based contrastive image-image pre-training for person re-identification on DeepSportRadar dataset \cite{van2022deepsportradar} having video frames of varying poses and camera angles of players (Fig. \ref{fig:clip_reident}). LCM-Clip-ReIdent models outperform original Clip-ReIdent and improve mean-average precision (mAP) across backbone  ViT-L-14 and ResNet50x16 (Table \ref{tab:lcm_clipreid})  for both approaches, with and without re-ranking. 
\begin{figure}[h]
  \centering
   \vspace{-5mm}
  \includegraphics[width=0.8\columnwidth]{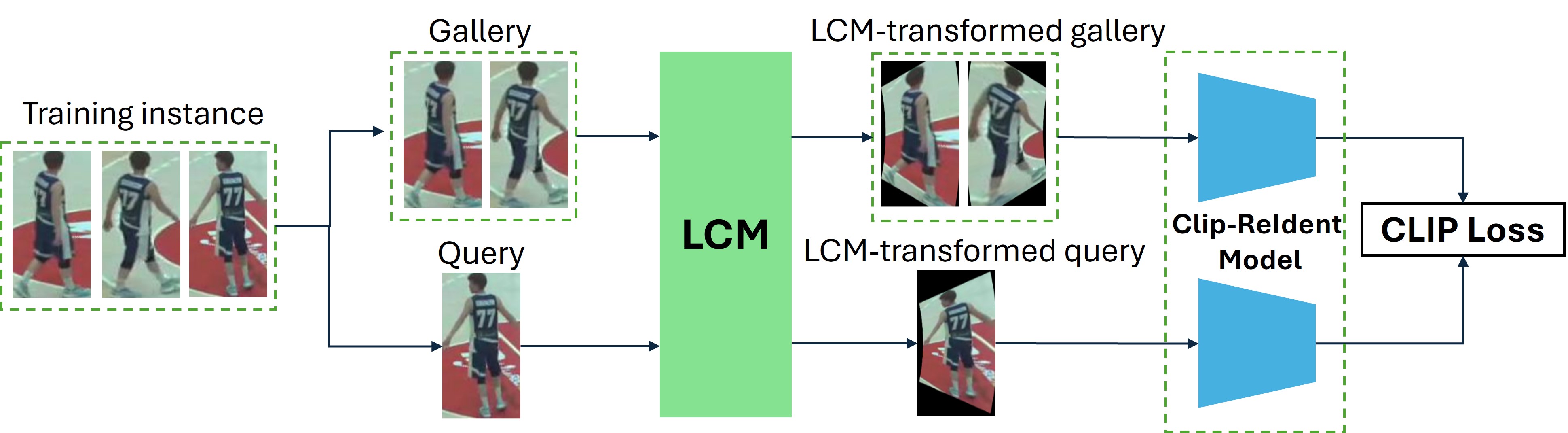}
  \caption{LCM-Clip-ReIdent: LCM transforms the input gallery and query during training of Clip-ReIdent for person re-id. LCM probability is set to 0.1.} 
  \label{fig:clip_reident}
  \vspace{-8mm}
\end{figure}

\begin{table}[h]
\vspace{-10mm}
\centering
\scriptsize
\caption{LCM improves Clip-ReIdent model.}
\begin{tabular}{cccc}
\hline
\textbf{Method} &
  \textbf{Encoder} &
  \textbf{\begin{tabular}[c]{@{}c@{}}mAP\\ (w/o Re-ranking)\end{tabular}} &
  \textbf{\begin{tabular}[c]{@{}c@{}}mAP\\ (with Re-ranking)\end{tabular}} \\ \hline
Baseline            & \multirow{3}{*}{ViT-L-14}    & 72.70          & -              \\
Clip-ReIdent        &                              & 96.90          & 98.20          \\
MPD-Clip-ReIdent &                              & 97.02 & 98.30 \\
LCM-Clip-ReIdent &                              & \textbf{97.16} & \textbf{98.42} \\\hline
Clip-ReIdent        & \multirow{2}{*}{ResNet50x16} & 88.50          & 94.90          \\
MPD-Clip-ReIdent &                              & \textbf{91.95} & 97.50 \\
LCM-Clip-ReIdent &                              & 91.90 & \textbf{97.62} \\\hline
\end{tabular}
\label{tab:lcm_clipreid}
\vspace{-0.3cm}
\end{table}

\vspace{-4mm}
\section{Conclusion and Future Work}
In this paper, we present the Log Conformal Maps (LCM) transform to efficiently handle perspective distortion, enhancing model robustness. Our analysis, backed by experiments, shows LCM's effectiveness in modeling perspective distortion. LCM-adapted models in both supervised and self-supervised settings surpass standard models and achieve state-of-the-art performance. Given its relevance to tasks like crowd counting and object detection, future directions include exploring LCM's application to diverse computer vision tasks and its computational advantages. We aim to inspire further research into using logarithmic conformal mapping for addressing geometric challenges in computer vision.\\
\textbf{Acknowledgment:} The authors thank Sumit Rakesh, Luleå University of Technology, for his support with the Lotty Bruzelius cluster. We also thank the National Supercomputer Centre at Linköping University for the Berzelius supercomputing, supported by the Knut and Alice Wallenberg Foundation.
\clearpage  

%
%
\bibliographystyle{splncs04}
\bibliography{main}

\newpage
\section{Supplementary Material}
\subsection{ImageNet-E}
In Table~\ref{tab:imagenet-e_detailed}, we report both the drop in accuracy and the absolute accuracy for the fully supervised ResNet-50 network and the self-supervised SimCLR~\cite{chen2020simple} method.

\begin{table}[h]
\tiny
\centering
\caption{ImageNet-E: (a) Drop of Top1 accuracy (Lower is better) and (b) Absolute Accuracy (Higher is better) under background changes, size changes, random position (rp), random direction (rd) categories. Results for ResNet50 are reported from MPD \cite{chhipa2024m}.}
\label{tab:imagenet-e_detailed}
\tiny
\begin{tabular}{ccccccccccccc}
\multicolumn{13}{c}{{\color[HTML]{000000} \textbf{(a) Drop of Accuracy}}}                                                                                                                               \\ \hline
\multicolumn{1}{c|}{{\color[HTML]{000000} }}                            & \multicolumn{1}{c|}{{\color[HTML]{000000} Original}} & \multicolumn{5}{c|}{{\color[HTML]{000000} Background changes}}                                                                                                                         & \multicolumn{4}{c|}{{\color[HTML]{000000} Size changes}}                                                                                       & \multicolumn{1}{c|}{{\color[HTML]{000000} Position}} & \multicolumn{1}{c}{{\color[HTML]{000000} Direction}}          \\ \cline{2-13}
\multicolumn{1}{c|}{\multirow{-2}{*}{{\color[HTML]{000000} Models}}}    & \multicolumn{1}{c|}{{\color[HTML]{000000} }}         & {\color[HTML]{000000} Inver} & {\color[HTML]{000000} \( \lambda \) = -20} & {\color[HTML]{000000} \( \lambda \) = 20} & {\color[HTML]{000000} \( \lambda \) = 20-adv} & \multicolumn{1}{c|}{{\color[HTML]{000000} Random}} & {\color[HTML]{000000} Full}  & {\color[HTML]{000000} 0.10}  & {\color[HTML]{000000} 0.08}  & \multicolumn{1}{c|}{{\color[HTML]{000000} 0.05}}  & \multicolumn{1}{c|}{{\color[HTML]{000000} rp}}       & \multicolumn{1}{c}{{\color[HTML]{000000} rd}}     \\ \hline
\multicolumn{1}{c|}{{\color[HTML]{000000} standard ResNet50}}           & \multicolumn{1}{c|}{{\color[HTML]{000000} 92.69}}    & {\color[HTML]{000000} 1.97}  & {\color[HTML]{000000} 7.30}    & {\color[HTML]{000000} 13.35}  & {\color[HTML]{000000} 29.92}      & \multicolumn{1}{c|}{{\color[HTML]{000000} 13.34}}  & {\color[HTML]{000000} 2.71}  & {\color[HTML]{000000} 7.25}  & {\color[HTML]{000000} 10.51} & \multicolumn{1}{c|}{{\color[HTML]{000000} 21.26}} & \multicolumn{1}{c|}{{\color[HTML]{000000} 26.46}}    & \multicolumn{1}{c}{{\color[HTML]{000000} 25.12}}                
\\ \hline
\multicolumn{1}{c|}{{\color[HTML]{000000} LCM (ResNet50)}} & \multicolumn{1}{c|}{{\color[HTML]{000000} 92.60}}    & {\color[HTML]{000000} 1.72}  & {\color[HTML]{000000} 6.70}    & {\color[HTML]{000000} 11.55}  & {\color[HTML]{000000} 29.95}      & \multicolumn{1}{c|}{{\color[HTML]{000000} 13.64}}  & {\color[HTML]{000000} 3.20}  & {\color[HTML]{000000} 6.08}  & {\color[HTML]{000000} 8.7}  & \multicolumn{1}{c|}{{\color[HTML]{000000} 18.54}} & \multicolumn{1}{c|}{{\color[HTML]{000000} 23.58}}    & \multicolumn{1}{c}{{\color[HTML]{000000} 22.92}}   \\
\multicolumn{1}{c|}{{\color[HTML]{000000} LCM+SimCLR (ResNet50)}}        & \multicolumn{1}{c|}{{\color[HTML]{000000} 92.55}}    & {\color[HTML]{000000} 1.40}  & {\color[HTML]{000000} 7.36}    & {\color[HTML]{000000} 12.23}  & {\color[HTML]{000000} 31.53}      & \multicolumn{1}{c|}{{\color[HTML]{000000} 14.15}}  & {\color[HTML]{000000} 2.96}  & {\color[HTML]{000000} 6.84}  & {\color[HTML]{000000} 10.57} & \multicolumn{1}{c|}{{\color[HTML]{000000} 20.55}} & \multicolumn{1}{c|}{{\color[HTML]{000000} 24.23}}    & \multicolumn{1}{c}{{\color[HTML]{000000} 23.39}}   \\
 \hline
\multicolumn{13}{c}{{\color[HTML]{000000} \textbf{(b) Absolute Accuracy}}}                                                                                                                                                                                                                                                                                                                                                                                                                                                                                                                                       \\ \hline
\multicolumn{1}{c|}{{\color[HTML]{000000} }}                            & \multicolumn{1}{c|}{{\color[HTML]{000000} Original}} & \multicolumn{5}{c|}{{\color[HTML]{000000} Background changes}}                                                                                                                         & \multicolumn{4}{c|}{{\color[HTML]{000000} Size changes}}                                                                                       & \multicolumn{1}{c|}{{\color[HTML]{000000} Position}} & \multicolumn{1}{c}{{\color[HTML]{000000} Direction}}                       \\ \cline{2-13}
\multicolumn{1}{c|}{\multirow{-2}{*}{{\color[HTML]{000000} Models}}}    & \multicolumn{1}{c|}{{\color[HTML]{000000} }}         & {\color[HTML]{000000} Inver} & {\color[HTML]{000000} \( \lambda \) = -20} & {\color[HTML]{000000} \( \lambda \) = 20} & {\color[HTML]{000000} \( \lambda \) = 20-adv} & \multicolumn{1}{c|}{{\color[HTML]{000000} Random}} & {\color[HTML]{000000} Full}  & {\color[HTML]{000000} 0.10}  & {\color[HTML]{000000} 0.08}  & \multicolumn{1}{c|}{{\color[HTML]{000000} 0.05}}  & \multicolumn{1}{c|}{{\color[HTML]{000000} rp}}       & \multicolumn{1}{c}{{\color[HTML]{000000} rd}} \\ \hline
\multicolumn{1}{c|}{{\color[HTML]{000000} standard ResNet50}}           & \multicolumn{1}{c|}{{\color[HTML]{000000} 92.69}}    & {\color[HTML]{000000} 90.72} & {\color[HTML]{000000} 85.39}   & {\color[HTML]{000000} 79.34}  & {\color[HTML]{000000} 62.77}      & \multicolumn{1}{c|}{{\color[HTML]{000000} 79.35}}  & {\color[HTML]{000000} 89.98} & {\color[HTML]{000000} 85.44} & {\color[HTML]{000000} 82.18} & \multicolumn{1}{c|}{{\color[HTML]{000000} 71.43}} & \multicolumn{1}{c|}{{\color[HTML]{000000} 66.23}}    & \multicolumn{1}{c}{{\color[HTML]{000000} 67.57}}              
\\ \hline
\multicolumn{1}{c|}{{\color[HTML]{000000} LCM (ResNet50)}} & \multicolumn{1}{c|}{{\color[HTML]{000000} 92.60}}    & {\color[HTML]{000000} 90.88} & {\color[HTML]{000000} 85.90}   & {\color[HTML]{000000} 81.05}  & {\color[HTML]{000000} 62.65}      & \multicolumn{1}{c|}{{\color[HTML]{000000} 78.96}}  & {\color[HTML]{000000} 89.40} & {\color[HTML]{000000} 86.52} & {\color[HTML]{000000} 83.90} & \multicolumn{1}{c|}{{\color[HTML]{000000} 74.06}} & \multicolumn{1}{c|}{{\color[HTML]{000000} 69.02}}    & \multicolumn{1}{c}{{\color[HTML]{000000} 69.68}}                   \\
\multicolumn{1}{c|}{{\color[HTML]{000000} LCM+SimCLR (ResNet50)}}        & \multicolumn{1}{c|}{{\color[HTML]{000000} 92.55}}    & {\color[HTML]{000000} 91.15} & {\color[HTML]{000000} 85.19}   & {\color[HTML]{000000} 80.32}  & {\color[HTML]{000000} 61.02}      & \multicolumn{1}{c|}{{\color[HTML]{000000} 78.40}}  & {\color[HTML]{000000} 89.59} & {\color[HTML]{000000} 85.71} & {\color[HTML]{000000} 81.98} & \multicolumn{1}{c|}{{\color[HTML]{000000} 72.00}} & \multicolumn{1}{c|}{{\color[HTML]{000000} 68.32}}    & \multicolumn{1}{c}{{\color[HTML]{000000} 69.16}}               \\ \hline
\end{tabular}
\vspace{-3mm}
\end{table}

\subsection{ImageNet-X}

Following the detailed results in Table~\ref{tab:imagenet_x_all}, it is evident that LCM-trained fully supervised ImageNet weights and self-supervised ImageNet weights offer better accuracies than the standard weights. Notably, LCM-trained weights demonstrate better or more competitive performance when compared to MPD-trained ImageNet weights.

\begin{table}[h]
\tiny
\centering
\caption{ImageNet-X: error ratios comparisons for MPD and LCM. Error ration close to 1.0 is ideal robustness for each factor. LCM demonstrate consistent robustness compared to standard ResNet50 model and matches the performance with MPD in supervised and self-supervised approaches.}
\label{tab:imagenet_x_all}
\begin{tabular}{c|ccc|cc}
\hline
\multirow{2}{*}{Factor} & \multicolumn{3}{c|}{Supervised}                      & \multicolumn{2}{c}{Self-supervsied} \\
                        & \multicolumn{1}{c|}{Standard ResNet50} & MPD  & LCM  & MPD+SimCLR       & LCM+SimCLR       \\ \hline
pose                    & \multicolumn{1}{c|}{0.70}              & 0.72 & 0.73 & 0.74             & 0.72             \\
pattern                 & \multicolumn{1}{c|}{0.77}              & 0.89 & 0.81 & 0.78             & 0.82             \\
partial\_view           & \multicolumn{1}{c|}{1.13}              & 0.89 & 0.86 & 0.86             & 0.87             \\
background              & \multicolumn{1}{c|}{1.14}              & 1.08 & 1.10 & 1.09             & 1.11             \\
brighter                & \multicolumn{1}{c|}{1.17}              & 1.09 & 1.10 & 1.09             & 1.09             \\
color                   & \multicolumn{1}{c|}{1.19}              & 1.21 & 1.17 & 1.18             & 1.16             \\
larger                  & \multicolumn{1}{c|}{1.32}              & 1.30 & 1.24 & 1.22             & 1.24             \\
smaller                 & \multicolumn{1}{c|}{1.50}              & 1.44 & 1.51 & 1.50             & 1.51             \\
darker                  & \multicolumn{1}{c|}{1.72}              & 1.52 & 1.50 & 1.51             & 1.52             \\
shape                   & \multicolumn{1}{c|}{1.75}              & 1.63 & 1.64 & 1.61             & 1.59             \\
object\_blocking        & \multicolumn{1}{c|}{1.92}              & 1.66 & 1.62 & 1.61             & 1.60             \\
style                   & \multicolumn{1}{c|}{1.93}              & 1.73 & 1.75 & 1.63             & 1.68             \\
multiple\_objects       & \multicolumn{1}{c|}{1.97}              & 1.76 & 1.72 & 1.72             & 1.69             \\
subcategory             & \multicolumn{1}{c|}{2.10}              & 1.92 & 1.94 & 1.90             & 1.84             \\
person\_blocking        & \multicolumn{1}{c|}{2.17}              & 1.95 & 1.98 & 1.97             & 2.10             \\
texture                 & \multicolumn{1}{c|}{2.39}              & 2.20 & 2.24 & 2.18             & 2.20             \\ \hline
Average                 & \multicolumn{1}{c|}{1.55}              & 1.44 & \textbf{1.43} & 1.41             & 1.42             \\ \hline
\end{tabular}

\end{table}

\end{document}